\documentclass{article}
\PassOptionsToPackage{numbers}{natbib}
\usepackage[preprint]{neurips_2026}
\setcitestyle{numbers,square,citesep={,}}
\usepackage[utf8]{inputenc}
\usepackage[T1]{fontenc}
\usepackage{hyperref}
\usepackage{url}
\usepackage{booktabs}
\usepackage{amsfonts}
\usepackage{amsmath}
\usepackage{amssymb}
\usepackage{amsthm}

\usepackage{mathtools}
\usepackage{nicefrac}
\usepackage{microtype}
\usepackage{xcolor}
\usepackage{multirow}
\usepackage{graphicx}
\usepackage{algorithm}
\usepackage{algpseudocode}
\usepackage{array}
\usepackage{subcaption}
\usepackage{enumitem}
\usepackage{float}
\usepackage{placeins}
\newtheorem{proposition}{Proposition}

\DeclareMathOperator{\KL}{KL}

\DeclareMathOperator{\sg}{sg}

\newcommand{\Ls}{\mathcal{L}_{\mathrm{s}}}
\newcommand{\Linfo}{\mathcal{L}_{\mathrm{info}}}
\newcommand{\LG}{\mathcal{L}_{\mathrm{G}}}
\newcommand{\LD}{\mathcal{L}_{\mathrm{D}}}
\newcommand{\LNPE}{\mathcal{L}_{\mathrm{NPE}}}
\newcommand{\LadvG}{\mathcal{L}^{G}_{\mathrm{adv}}}
\newcommand{\LadvD}{\mathcal{L}^{D}_{\mathrm{adv}}}
\newcommand{\Lcyc}{\mathcal{L}_{\mathrm{cyc}}}
\newcommand{\Lid}{\mathcal{L}_{\mathrm{id}}}

\title{Information-Preserving Domain Transfer with Unlabeled Data in Misspecified Simulation-Based Inference}

\author{%
\parbox{\textwidth}{\centering
Joon Jang\textsuperscript{1} \quad
Eunho Jeong\textsuperscript{2} \quad
Kyu Sung Choi\textsuperscript{3,4,5}\thanks{Corresponding authors.} \quad
Hyeonjin Kim\textsuperscript{3,6}\footnotemark[1]
\\[1mm]
{\normalfont
\textsuperscript{1}Department of Biomedical Sciences, Seoul National University, Seoul, Republic of Korea\\
\textsuperscript{2}Department of Applied Bioengineering, Graduate School of Convergence Science and Technology, Seoul National University, Seoul, Republic of Korea\\
\textsuperscript{3}Department of Radiology, Seoul National University Hospital, Seoul, Republic of Korea\\
\textsuperscript{4}Department of Radiology, Seoul National University College of Medicine, Seoul, Republic of Korea\\
\textsuperscript{5}Healthcare AI Research Institute, Seoul National University Hospital, Seoul, Republic of Korea\\
\textsuperscript{6}Department of Medical Sciences, Seoul National University, Seoul, Republic of Korea
}
}
}

\begin{document}
\maketitle

\begin{abstract}
Simulation-based inference (SBI) provides amortized Bayesian parameter inference
from simulator-generated data without requiring explicit likelihood evaluation. Its reliability can
degrade under model misspecification, where real-world observations are not well
represented by the simulator used for training. Existing methods using unlabeled real-world data often align simulated and
real-world data distributions, but marginal alignment alone does not directly
preserve parameter-relevant information needed for posterior inference. We
propose SPIN, an SBI framework with parameter-relevant information-preserving
domain transfer using unlabeled, unpaired real-world observations. During training, SPIN translates labeled simulator observations toward the
real-world domain and back to the simulator domain, using the original simulator
labels to encourage domain transfer that preserves parameter-relevant mutual
information. At test time, the learned real-to-simulator
transport maps real-world observations into the simulator domain for posterior
inference, without requiring real-world parameter labels or paired
real--simulator observations. Across controlled synthetic and physical real-world benchmarks, SPIN improves
real-world posterior inference, with the improvement becoming clearer as
misspecification increases.
\end{abstract}

\section{Introduction}

Simulation-based inference (SBI) is widely used for Bayesian parameter inference
across scientific and engineering applications where simulators are available
but explicit likelihood evaluation is inconvenient, expensive, or inaccessible
\citep{cranmer2020frontier,deistler2025practicalguide}. Modern neural SBI methods train an inference network on simulated data to learn
the probabilistic relationship between simulation outputs and parameters. Neural
posterior estimation (NPE), a common neural SBI approach, uses a conditional
generative model to directly approximate the posterior distribution, enabling
amortized inference for new observations
\citep{papamakarios2016npe,greenberg2019automatic}. This makes NPE useful for
repeated inference, but its reliability remains sensitive to \textit{model
misspecification}: when real-world observations are not well represented by the
simulator, the posterior estimator is evaluated outside the distribution used for
training \citep{box1976science,frazier2020abc,cannon2022misspecification}. Its consequences and diagnostics have been
studied in approximate Bayesian computation, Bayesian asymptotics under
misspecification, and neural SBI, including posterior bias, overconfidence,
out-of-simulation behavior, and statistical model criticism
\citep{frazier2020abc,cannon2022misspecification,kleijn2012bernstein,gelman1996posterior,
ward2022rnpe,schmitt2024detecting,huang2023robuststatistics}. In this setting, direct amortized
inference on real-world data can yield biased or overconfident posteriors
\citep{cannon2022misspecification,ward2022rnpe,mishra2025selfconsistency}.

Recent approaches to misspecified SBI differ in the type of real-world
supervision they require. RoPE \citep{wehenkel2025rope} targets settings with a small real-world
calibration set of observations with known parameters and formalizes the
misspecification gap as an optimal-transport problem
\citep{courty2017optimal,villani2009optimal,peyre2019computational} between learned
representations of real-world and simulated observations.
FRISBI \citep{senouf2025frisbi} extends this known-parameter real-world data setting to an inductive and
amortized setting using OT-based
alignment over unpaired observations.

Without such a calibration set, recent work has used unlabeled real-world
observations during training. NPE with self-consistency (NPE+SC) \citep{mishra2025selfconsistency} augments the
simulation-based posterior loss with a self-consistency loss evaluated on
unlabeled observations. This loss exploits the
Bayesian self-consistency relation \citep{schmitt2024selfconsistency,ivanova2024selfconsistency} that the likelihood--prior to posterior ratio
is constant across parameter values under exact inference. However, evaluating
this constraint requires likelihood evaluations or an auxiliary likelihood
estimator.

Other methods using unlabeled real-world data include domain alignment
methods such as NPE with robust statistics (NPE-RS) \citep{huang2023robuststatistics}, NPE-MMD \citep{elsemuller2025uda}, and NPE-DANN
\citep{elsemuller2025uda}. NPE-RS and NPE-MMD use
maximum mean discrepancy (MMD) \citep{gretton2012mmd} alignment, while NPE-DANN uses
domain-adversarial neural network (DANN) alignment \citep{ganin2016dann}. These
methods show that domain alignment with unlabeled data can improve robustness under model
misspecification, but their adaptation objectives primarily match marginal
summary distributions or induce domain confusion in summary space.
The underlying issue is that marginal alignment does not determine the conditional
structure relating observations to task variables. This limitation is closely related
to results in domain adaptation, where marginal invariance does not ensure
conditional invariance and non-invertible representations can lose information
needed for prediction \citep{johansson2019support,zhao2019learning,stojanov2021domain}. Similarly, translation-based domain adaptation
has used semantic-consistency constraints because aligning marginal distributions
does not enforce semantic consistency and may change label-relevant content \citep{hoffman2018cycada}. Thus,
marginal domain alignment is useful but only indirectly constrains the
information needed for posterior inference.

We therefore propose SPIN, an SBI framework with parameter-relevant
information-preserving domain transfer for NPE under model misspecification.
SPIN learns bidirectional simulator--real-world transports using a
variational lower-bound objective designed to preserve parameter-relevant information
after transport. During training, a labeled simulator
observation is translated toward the real-world domain and then mapped back to
the simulator domain. Because this transported observation originates from a
labeled simulator sample, the original simulator parameter label remains
available, allowing SPIN to optimize the information-preserving constraint
without real-world parameter labels or paired real--simulator observations. At
test time, the learned real-to-simulator transport maps each real-world
observation into the simulator domain before posterior evaluation. Our main contributions are as follows:

\begin{enumerate}
\item \textbf{Information-preserving domain transfer.}
SPIN constrains domain transfer to preserve information needed for posterior
inference, rather than relying only on marginal alignment between simulated and
real-world data.

\item \textbf{Unlabeled, unpaired real-world adaptation.}
SPIN requires only simulator labels for training, enabling NPE adaptation with
unlabeled, unpaired real-world observations without real-world parameter labels.
\end{enumerate}

Across synthetic and controlled physical real-world benchmarks, SPIN improves
the main inference-quality metrics, with larger gains as the simulator--reality
gap increases.

\section{Method}

\subsection{Problem setup}

Let $\theta \in \Theta$ denote the parameter of interest and let
$x_s \in \mathcal{X}_s$ denote a simulated observation generated from
$p_s(x_s \mid \theta)$ with prior $p(\theta)$. Let
$x_r \in \mathcal{X}_r$ denote a real-world observation. We assume access to
$N_s$ labeled simulated pairs
$\mathcal{D}_s=\{(\theta_i^s,x_{s,i})\}_{i=1}^{N_s}$ and $N_r$ unlabeled
real-world observations $\mathcal{D}_r=\{x_{r,j}\}_{j=1}^{N_r}$. Real-world
observations are available without parameter labels and without pairing to
individual simulated samples.

NPE learns a conditional density estimator from simulated
parameter--observation pairs \citep{papamakarios2016npe,greenberg2019automatic}.
This estimator is often parameterized as a conditional normalizing flow, and for
high-dimensional observations it is commonly conditioned on a neural statistic
estimator $h_\omega$ \citep{radev2020bayesflow,chen2021nass}:
\[
q_{\psi,\omega}(\theta \mid x_s)
:=
q_\psi(\theta \mid h_\omega(x_s)).
\]
For notational simplicity, we write $q_\psi(\theta \mid x)$ throughout the
paper, with the understanding that this denotes
$q_\psi(\theta \mid h_\omega(x))$ when a summary network is used. Standard NPE minimizes the negative conditional log-likelihood of simulator
parameters under the posterior density estimator
\citep{papamakarios2016npe,greenberg2019automatic}. We denote this
simulator-side conditional density objective for a generic estimator $q$ by
\begin{equation*}
\Ls(q)
:=
\mathbb{E}_{(\theta,x_s)\sim p_s(\theta,x_s)}
\left[
-\log q(\theta \mid x_s)
\right].
\end{equation*}

Under model misspecification, direct evaluation of a simulator-trained posterior
on $x_r$ can be unreliable because real-world observations may not be represented
by the simulator distribution \citep{cannon2022misspecification,ward2022rnpe,mishra2025selfconsistency}. SPIN therefore defines test-time inference after
real-to-simulator transport:
\begin{equation}
\hat p(\theta \mid x_r)
:=
q_\psi(\theta \mid x_{rs}),
\qquad
x_{rs}:=G_{rs}(x_r).
\label{eq:test_rule_new}
\end{equation}
Here, $G_{rs}:\mathcal{X}_r\rightarrow\mathcal{X}_s$ denotes the
real-to-simulator transport. We also introduce the reverse transport
$G_{sr}:\mathcal{X}_s\rightarrow\mathcal{X}_r$, which maps simulator
observations toward the real-world domain. The two transports define a
bidirectional observation-space translation \citep{zhu2017cyclegan}: $G_{sr}$
produces observations in the real-world domain from simulator samples, and
$G_{rs}$ maps real-world observations into the simulator domain used by the
posterior estimator. For reliable test-time inference, $G_{rs}$ should transport
observations to the simulator domain while preserving parameter-relevant
information.

\subsection{Parameter-relevant information-preserving domain transfer}

As an idealized reference, suppose that real-world parameter labels were
available. Then preservation of parameter-relevant information under the
test-time transport $G_{rs}$ could be expressed by maximizing the mutual information (MI) \citep{shannon1948mathematical,cover2006elements}:
\begin{equation*}
\max_{G_{rs}} I_r\!\big(\Theta; G_{rs}(X_r)\big).
\end{equation*}
Here, $I_r(\cdot;\cdot)$ denotes MI under the real-world joint
distribution. In the unlabeled real-world setting, this objective is not
available. We therefore approach the information-maximization problem using transported
simulator observations that retain their simulator labels.

For a simulated pair $(\theta,x_s)$, define
\[
x_{sr} := G_{sr}(x_s),
\qquad
x_{srs} := G_{rs}(x_{sr}) = G_{rs}(G_{sr}(x_s)).
\]
The observation $x_s$ is first translated toward the real-world domain and then
mapped back to the simulator domain. Because $x_{srs}$ originates from a
labeled simulator sample, the original simulator parameter label remains
available. The information transmitted through this
transported observation can be quantified by the MI
\[
I_s\!\left(\Theta;G_{rs}(G_{sr}(X_s))\right)
=
H_s(\Theta)
-
H_s\!\left(\Theta\mid G_{rs}(G_{sr}(X_s))\right).
\]
Here, $H_s$ denotes entropy under the simulator-induced distribution. The
conditional entropy term is not directly available, but it can be bounded using
a variational approximation $q(\Theta\mid G_{rs}(G_{sr}(X_s)))$. By the non-negativity of the conditional Kullback--Leibler (KL) divergence,
\citep{barber2003algorithm,poole2019variational}:
\begin{equation}
I_s\!\left(\Theta;G_{rs}(G_{sr}(X_s))\right)
\ge
H_s(\Theta)
+
\mathbb{E}_{(\Theta,X_s)\sim p_s}
\!\left[
\log q\!\left(\Theta\mid G_{rs}(G_{sr}(X_s))\right)
\right].
\label{eq:mi_lower_bound}
\end{equation}

The bound is exact when
$q(\Theta\mid G_{rs}(G_{sr}(X_s)))$ equals the true conditional posterior. The
expectation term increases when the variational posterior assigns higher density
to the original simulator parameter after transport. This leads to the information-preservation loss, defined as the negative of the
expectation term in Eq.~\eqref{eq:mi_lower_bound}:
\begin{equation}
\Linfo(q)
:=
\mathbb{E}_{(\Theta,X_s)\sim p_s}
\Big[
-\log q\big(\Theta \mid G_{rs}(G_{sr}(X_s))\big)
\Big].
\label{eq:linfo_new}
\end{equation}
Since $H_s(\Theta)$ is fixed, minimizing $\Linfo$ maximizes the lower bound
above. During transport training, this objective updates $G_{rs}$ and $G_{sr}$
so that the transported simulator observation remains informative about the
original parameter.

If the learned transports match the observation distributions across domains
and maintain the statistical dependence encouraged by $\Linfo$, the
simulator-to-real component satisfies the population-level condition
\begin{equation}
(\Theta, X_r)
\stackrel{d}{=}
(\Theta, G_{sr}(X_s)),
\qquad
(\Theta,X_s)\sim p_s(\theta,x_s).
\label{eq:joint_sim_to_real_condition}
\end{equation}
That is, $G_{sr}(X_s)$ is matched
to $X_r$ jointly with the parameter. Under this condition, the learned real-to-simulator transport $G_{rs}$ can be
applied to real-world observations $x_r$ at test time while preserving
parameter-relevant information. 
This corresponds to the inference rule in Eq.~\eqref{eq:test_rule_new}. Appendix~\ref{app:info_math_details} formalizes this connection by showing that, under Eq.~\eqref{eq:joint_sim_to_real_condition}, the real-world test-time risk equals the simulator-labeled transported risk used by $L_{\mathrm{info}}$.

\subsection{Training objective}

SPIN jointly trains the transport networks, discriminators, and posterior
estimator. The full training procedure is provided in
Algorithm~\ref{alg:spin_training} in Appendix~\ref{app:implementation}.

\paragraph{Generator loss.}
The generator loss combines adversarial translation, cycle-consistency, identity
regularization \citep{zhu2017cyclegan}, and information preservation:
\begin{equation*}
\LG
=
\lambda_{\mathrm{adv}}\LadvG
+
\lambda_{\mathrm{cyc}}\Lcyc
+
\lambda_{\mathrm{id}}\Lid
+
\lambda_{\mathrm{info}}
\mathbb{E}_{(\theta,x_s)\sim p_s(\theta,x_s)}
\left[
-\log q_{\sg(\psi,\omega)}\!\left(\theta \mid x_{srs}\right)
\right].
\end{equation*}
Here, $\sg(\cdot)$ denotes the stop-gradient operator, and
$q_{\sg(\psi,\omega)}$ indicates that the summary network and posterior
estimator are held fixed during the generator update. Thus, the
information-preservation term updates only the transport networks $G_{rs}$ and
$G_{sr}$.

\paragraph{Discriminator loss.}
The discriminators use hinge adversarial losses with spectral normalization
\citep{miyato2018spectral}:
\[
\LD=\LadvD.
\]

\paragraph{Posterior loss.}
The posterior estimator is trained on both original simulated observations and
transported observations with simulator labels:
\begin{equation*}
\LNPE
=
\Ls(q_\psi)
+
\lambda_{\mathrm{info}}
\mathbb{E}_{(\theta,x_s)\sim p_s(\theta,x_s)}
\left[
-\log q_\psi\!\left(\theta \mid \sg(x_{srs})\right)
\right].
\end{equation*}
Here, $\sg(x_{srs})$ keeps the transported observations fixed. This update
extends posterior training to transported simulator-domain inputs, while
preventing the posterior loss from changing the transport itself.

Unless otherwise stated, we set $\lambda_{\mathrm{info}}=1$. This assigns the
transported observation with a simulator label the same negative conditional
log-likelihood weight as the original simulator observation, without selecting
the constraint strength using real-world labels.  Figure~\ref{fig:method_overview}
summarizes the training and inference procedure.

\begin{figure}[t]
    \centering
    \includegraphics[width=0.92\linewidth]{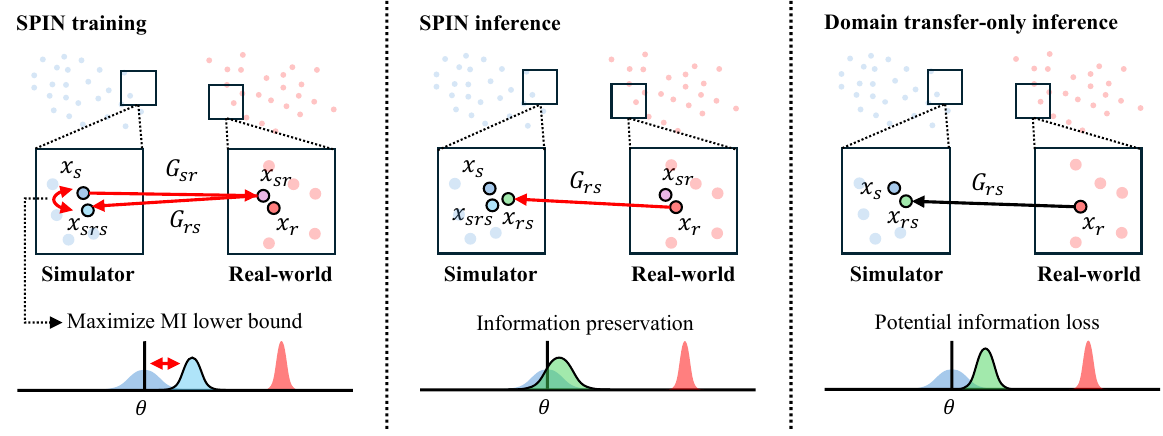}
    \caption{
\textbf{Overview of SPIN.}
During training, a labeled simulator observation $(\theta,x_s)$ is translated toward the real-world domain as $x_{sr}$ and returned to the simulator domain as $x_{srs}$. Since $x_{srs}$ retains the simulator-origin label, SPIN uses $\theta$ to maximize the MI lower bound in Eq.~\eqref{eq:mi_lower_bound}, encouraging parameter-relevant information to be preserved after transport. At test time, an unlabeled real-world observation $x_r$ is mapped to the simulator domain as $x_{rs}=G_{rs}(x_r)$, and inference is performed as $q_\psi(\theta\mid x_{rs})$. The right panel illustrates that domain transfer alone can match domains while losing information needed for posterior inference.
}
    \label{fig:method_overview}
\end{figure}

\section{Experiments}

We evaluate SPIN in the unlabeled, unpaired real-world adaptation regime.
Real-world parameter labels, when available in controlled benchmarks, are used
only for evaluation.

\subsection{Benchmarks}

Additional dataset and simulation details are provided in
Appendix~\ref{app:datasets}.

\paragraph{SIR.}
SIR \citep{kermack1927sir} is a weakly misspecified synthetic benchmark used in recent robust SBI
evaluations and based on a discrete-time stochastic susceptible--infected--recovered model \citep{wehenkel2025rope,senouf2025frisbi,elsemuller2025uda}. The
parameter is
\[
\theta = (\beta, \gamma),
\]
where $\beta$ denotes the infection rate and $\gamma$ denotes the recovery rate.
Misspecification arises from delayed weekend reporting, where a fraction of
Saturday and Sunday infections is shifted to Monday.

\paragraph{Pendulum.}
Pendulum is a controlled synthetic benchmark used in prior misspecified SBI
studies \citep{wehenkel2025rope,senouf2025frisbi}. The synthetic real-world
domain introduces damping while preserving the same parameter value
\[
\theta=(\omega_0,\phi),
\]
where $\omega_0$ denotes the natural angular frequency and $\phi$ denotes the
initial phase. The damping coefficient $\alpha$ is sampled from a continuous
range so that the benchmark reflects a continuum of misspecification.

\paragraph{Wind Tunnel.}
Wind Tunnel evaluates a controlled physical simulator--reality gap in a
wind-chamber system. This benchmark uses the controlled wind-tunnel setup from
Causal Chambers \citep{gamella2025causalchambers}, which has been used for misspecified SBI evaluation
\citep{wehenkel2025rope,senouf2025frisbi}. Following
the benchmark setup, the inference target is
\[
\theta = H,
\]
where $H$ denotes the hatch position controlling the wind-tunnel opening. The
simulator captures the main physical mechanism, whereas misspecification arises
because the mechanistic model only approximates the real chamber response to fan
loads and hatch position.

\paragraph{Light Tunnel.}
Light Tunnel also evaluates a controlled physical simulator--reality gap, here
in an optical imaging system. This benchmark uses the controlled light-tunnel
setup from Causal Chambers \citep{gamella2025causalchambers}, which has also been used in prior misspecified SBI
benchmarks
\citep{wehenkel2025rope,senouf2025frisbi}. The
inference target is
\[
\theta=(R,G,B,\alpha),
\]
where $(R,G,B)$ denote the light-source color intensities and $\alpha$ denotes
the polarizer-induced dimming factor. The simulator captures the dominant
optical mechanism, whereas misspecification arises from simplified modeling of
the real light-tunnel response.

\subsection{Baselines}

We compare SPIN to standard NPE \citep{papamakarios2016npe,greenberg2019automatic} and two unlabeled domain-alignment methods,
NPE-MMD \citep{elsemuller2025uda} and NPE-DANN \citep{elsemuller2025uda}. NPE uses only labeled simulations, whereas NPE-MMD and
NPE-DANN also use unlabeled real-world observations for adaptation. We use
NPE-MMD rather than NPE-RS \citep{huang2023robuststatistics} because the original NPE-RS formulation is not fully
amortized. Across methods, the same summary network and posterior flow are used
within each task.

RoPE \citep{wehenkel2025rope} and FRISBI \citep{senouf2025frisbi} are not included as baselines because they use a small
real-world calibration set with known parameters. NPE+SC \citep{mishra2025selfconsistency} is also not included as
a baseline because it relies on Bayesian self-consistency \citep{schmitt2024selfconsistency,ivanova2024selfconsistency} with likelihood
evaluations or an auxiliary likelihood estimator. Related methods are summarized in
Appendix~\ref{app:method_comparison}.
Implementation details, architectures, and
optimization settings are provided in Appendix~\ref{app:implementation}.

\subsection{Evaluation protocol}

Real-world parameter labels are reserved exclusively for evaluation and are
never used during adaptation. Unless otherwise stated, quantitative results are
reported as mean $\pm$ standard deviation over five independent training runs
with different random seeds. Posterior means and sample-based
posterior-discrepancy metrics are computed using $1000$ posterior samples per
observation.

\paragraph{Main metrics.}
We evaluate real-world posterior quality using root mean squared error (RMSE),
log posterior probability (LPP), and ACAUC, defined as the signed area between
the empirical coverage curve and the nominal coverage diagonal \citep{wehenkel2025rope,senouf2025frisbi}. RMSE of posterior means measures point-estimation
accuracy, with lower values indicating that posterior mass is centered closer to
the ground-truth parameter. LPP evaluates the posterior density assigned to the
ground-truth parameter \citep{lueckmann2021benchmarking}. It is the evaluation metric closest in form
to the information-preservation objective, because both LPP and $\Linfo$ depend
on the posterior density assigned to the corresponding parameter value. From the
variational information-maximization perspective, this posterior-density term
acts as the tractable energy term in a lower bound on parameter-relevant mutual
information, with the entropy term fixed within each task
\citep{barber2003algorithm,poole2019variational}. Thus, higher LPP can be
interpreted as stronger posterior-based information transmission from the
posterior input to the parameter, up to the variational approximation described
in Appendix~\ref{app:LPPeq}.
ACAUC is used as a scalar posterior diagnostic based on expected-coverage
analysis \citep{hermans2022trust}. We interpret ACAUC together with RMSE and LPP, because
global diagnostics provide only a partial check of posterior quality. The
corresponding diagnostic curves are reported in
Appendix~\ref{app:calibration_analysis}. Appendix~\ref{app:xsrs_supervised_eval} further evaluates the
simulator-originated transported observations $x_{srs}$, where $\Linfo$ is
directly applied, using RMSE, LPP, and ACAUC. To assess sensitivity to
misspecification strength, we vary the Pendulum damping coefficient and evaluate
RMSE, LPP, and ACAUC across damping levels.

We report symmetric KL divergence (sKL) \citep{kullback1951information}, Wasserstein distance \citep{ramdas2017wasserstein}, and posterior MMD \citep{gretton2012mmd} together in
Appendix~\ref{app:additional_quant} as secondary posterior-discrepancy
diagnostics, computed between simulator-side posteriors and real-world posteriors for matched
simulated and real-world observations.

To evaluate the contribution of parameter-relevant information preservation, we
include a transport-only variant, SPIN (w/o $\Linfo$), which keeps the same
bidirectional transport architecture but removes the information-preservation
term. Comparing this variant with SPIN tests whether simulator-labeled
conditional density supervision on transported observations improves real-world
posterior quality.

\paragraph{Ablation study.}
We vary $N_r$ to assess the dependence of SPIN
on the amount of unlabeled real-world data used for adaptation. This analysis
evaluates sensitivity to the size of the real-world adaptation set, including
settings with limited real-world observations. We additionally vary
$\lambda_{\mathrm{info}}$ in Appendix~\ref{app:lambda_sensitivity}
(Figure~\ref{fig:lambda_info_sensitivity}) to examine the sensitivity of the
method to the strength of information-preservation supervision.

\section{Results and Discussion}

\subsection{Main results}

\begin{table}[t]
\centering
\caption{Main benchmark results on the four tasks. Results are reported as mean $\pm$ standard deviation over runs. Lower is better for RMSE, higher is better for LPP, and ACAUC is best when closer to zero.}
\label{tab:main_results}
\resizebox{\linewidth}{!}{
\begin{tabular}{llccccc}
\toprule
Task & Metric & NPE & NPE-MMD & NPE-DANN & SPIN (w/o $\Linfo$) & SPIN \\
\midrule
\multirow{3}{*}{SIR}
& RMSE  & $0.059 \pm 0.005$ & $\mathbf{0.058 \pm 0.001}$ & $0.059 \pm 0.002$ & $0.060 \pm 0.006$ & $0.059 \pm 0.003$ \\
& LPP   & $4.002 \pm 0.226$ & $4.043 \pm 0.463$ & $\mathbf{4.235 \pm 0.281}$ & $3.350 \pm 1.797$ & $3.775 \pm 0.757$ \\
& ACAUC & $-0.032 \pm 0.014$ & $-0.033 \pm 0.022$ & $-0.048 \pm 0.010$ & $\mathbf{-0.011 \pm 0.029}$ & $-0.025 \pm 0.016$ \\
\midrule

\multirow{3}{*}{Pendulum}
& RMSE  & $0.774 \pm 0.217$ & $0.514 \pm 0.140$ & $0.572 \pm 0.097$ & $0.344 \pm 0.069$ & $\mathbf{0.287 \pm 0.062}$ \\
& LPP   & $-49.89 \pm 19.17$ & $-89.53 \pm 68.95$ & $-57.02 \pm 72.97$ & $-8.341 \pm 6.824$ & $\mathbf{-4.327 \pm 3.914}$ \\
& ACAUC & $0.265 \pm 0.078$ & $0.220 \pm 0.058$ & $0.185 \pm 0.053$ & $0.077 \pm 0.063$ & $\mathbf{0.062 \pm 0.063}$ \\
\midrule

\multirow{3}{*}{Wind Tunnel}
& RMSE  & $7.768 \pm 0.893$ & $4.108 \pm 0.928$ & $4.345 \pm 0.654$ & $4.128 \pm 0.438$ & $\mathbf{4.098 \pm 0.433}$ \\
& LPP   & $-43.02 \pm 16.61$ & $-26.33 \pm 10.62$ & $-27.43 \pm 14.42$ & $-26.96 \pm 8.059$ & $\mathbf{-10.04 \pm 1.412}$ \\
& ACAUC & $0.417 \pm 0.013$ & $0.304 \pm 0.052$ & $0.317 \pm 0.039$ & $0.344 \pm 0.025$ & $\mathbf{0.280 \pm 0.020}$ \\
\midrule

\multirow{3}{*}{Light Tunnel}
& RMSE  & $66.16 \pm 11.38$ & $57.51 \pm 2.727$ & $47.56 \pm 7.008$ & $40.41 \pm 3.226$ & $\mathbf{35.33 \pm 0.871}$ \\
& LPP ($\times 10^3$) & $-3.626 \pm 2.522$ & $-2.394 \pm 0.221$ & $-1.404 \pm 0.649$ & $-0.536 \pm 0.087$ & $\mathbf{-0.124 \pm 0.013}$ \\
& ACAUC & $0.326 \pm 0.027$ & $0.246 \pm 0.019$ & $0.211 \pm 0.052$ & $0.187 \pm 0.017$ & $\mathbf{0.163 \pm 0.017}$ \\
\bottomrule
\end{tabular}
}
\end{table}

Table~\ref{tab:main_results} summarizes the quantitative results across the four
benchmarks. Overall, SPIN showed a more stable performance profile across RMSE,
LPP, and ACAUC than the marginal domain-alignment baselines.
Appendix~\ref{app:qualitative_transport_examples} provides qualitative transport
examples as visual diagnostics of the learned observation-space transport.

NPE-MMD and NPE-DANN improved over standard NPE in some settings, but their gains
were not uniform across tasks and metrics. In comparison, SPIN improved the main
posterior-quality metrics more consistently on Pendulum, Wind Tunnel, and Light
Tunnel. This indicates that real-world posterior inference benefits not only
from reducing domain discrepancy, but also from domain transport that preserves
information relevant to the parameter.

This effect is most directly reflected in LPP. The LPP gains on Pendulum, Wind Tunnel, and Light
Tunnel therefore support the intended information-preserving role of
real-to-simulator transport. The accompanying RMSE and ACAUC gains indicate that the LPP improvement is not
only a density-level effect. Posterior means also move closer to the reference
parameters, while the coverage diagnostic shows improved uncertainty behavior. Appendix~\ref{app:xsrs_supervised_eval} reports the corresponding check on
$x_{srs}$, where the effect of $\Linfo$ should be most directly observed. The
higher LPP of SPIN on this path confirms that $\Linfo$ behaves as intended on
the labeled simulator-originated observations.

SIR showed smaller differences among methods. This is consistent with a boundary
case where standard NPE already provides a strong reference point, leaving less
room for adaptation to improve posterior inference. Similar behavior has been
reported in FRISBI, where standard NPE can be sufficient under minimal
misspecification \citep{senouf2025frisbi}. In such settings, additional
alignment or transport objectives may yield smaller or less consistent gains
because adaptation introduces optimization and regularization trade-offs
\citep{elsemuller2025uda}.

On Pendulum, Wind Tunnel, and Light Tunnel, adding $\Linfo$ improved the main
inference metrics compared with SPIN (w/o $\Linfo$). This pattern indicates that
the improvement is not only due to bidirectional transport, but also to the
information-preservation constraint imposed by $\Linfo$, which encourages
transported observations to remain informative about the corresponding simulator
parameter.

The posterior-discrepancy diagnostics in Appendix~\ref{app:additional_quant}
provide secondary support for this interpretation. SPIN reduced sKL on the
benchmarks with larger simulator--reality gaps, indicating reduced
posterior-level disagreement between matched simulator and real-world
observations. Wasserstein distance and posterior MMD showed broadly consistent
but task-dependent patterns, suggesting that this effect is not limited to a
single discrepancy measure.

\subsection{Sensitivity to misspecification strength}
\label{sec:misspec_strength}

The Pendulum benchmark allows the strength of model misspecification to be
controlled through the damping coefficient. Figure~\ref{fig:posterior_compare}
shows qualitative posterior comparisons on matched Pendulum examples under
increasing misspecification. As damping increased, direct inference on the real-world observation produced
posteriors with less mass near the reference parameter, whereas SPIN placed
posterior mass closer to the reference parameter after real-to-simulator
transport.

\begin{figure}[t]
    \centering
    \includegraphics[width=0.92\linewidth]{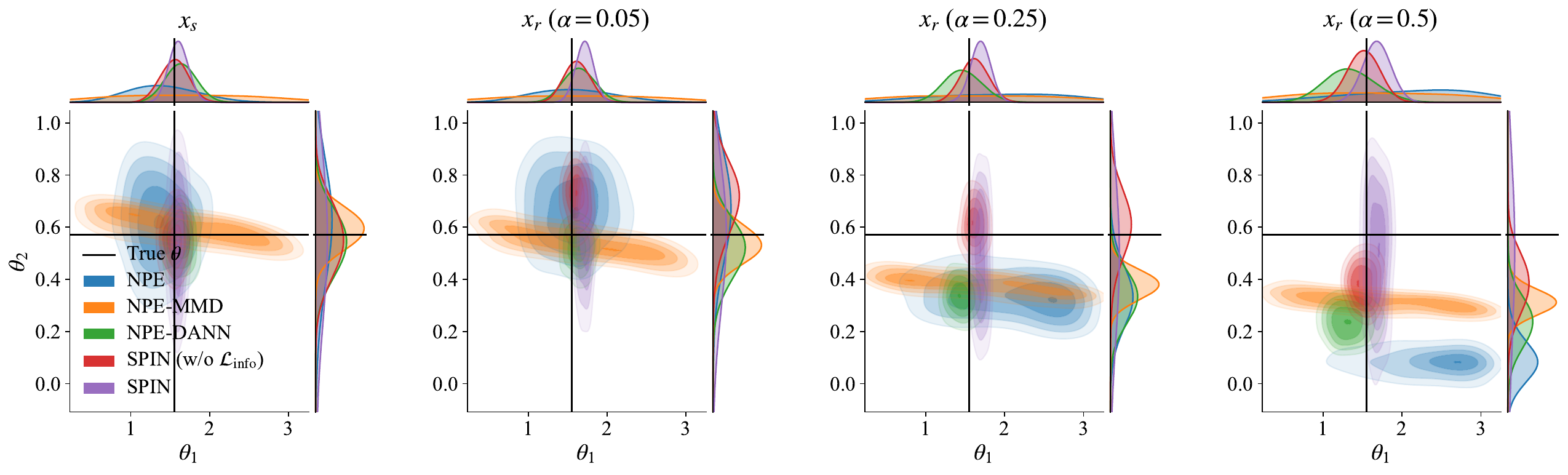}
    \caption{\textbf{Posterior comparison across methods on matched Pendulum samples under increasing misspecification.}
    Each column shows a matched simulator/real-world example sharing the same
    parameter, with the real-world observation generated under a
    different damping strength $\alpha$. In
    the stronger damping settings, SPIN shows posterior mass closer to the
    reference parameter. These examples provide qualitative posterior comparisons under increasing
    misspecification and complement the RMSE, LPP, and ACAUC results in
    Table~\ref{tab:main_results}.}
    \label{fig:posterior_compare}
\end{figure}

Figure~\ref{fig:misspec_strength} summarizes the corresponding quantitative
trend. Under weak shift, all adaptation methods were relatively close, while SPIN
remained competitive. As damping increased, the gains of SPIN became clearer,
particularly in LPP while maintaining favorable RMSE and ACAUC. Together, the posterior examples and quantitative
metrics indicate that the benefit of SPIN becomes clearer as misspecification
increases. This within-task trend is consistent with the cross-benchmark pattern in
Table~\ref{tab:main_results}, where SPIN showed clearer improvements on
Pendulum and on the controlled physical simulator--reality gaps in Wind Tunnel
and Light Tunnel than on weakly misspecified SIR.

\begin{figure}[t]
    \centering
    \includegraphics[width=0.92\linewidth]{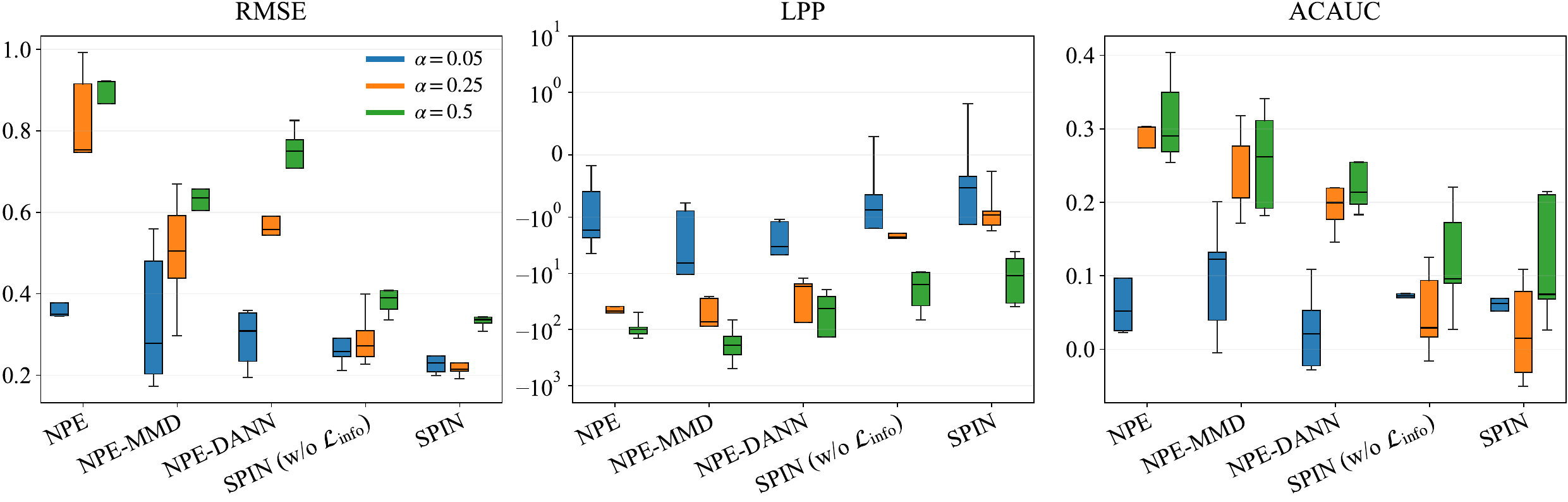}
    \caption{\textbf{Sensitivity to misspecification strength on Pendulum.}
    Box plots summarize performance over five independent runs at three damping
    levels $\alpha \in \{0.05, 0.25, 0.5\}$. Under weak shift, methods are relatively close, while SPIN shows clearer gains as
damping increases. SPIN (w/o $\mathcal{L}_{\mathrm{info}}$) removes
    simulator-labeled information-preservation supervision, highlighting the
    contribution of $\mathcal{L}_{\mathrm{info}}$ under stronger
    misspecification.}
    \label{fig:misspec_strength}
\end{figure}

The comparison with SPIN (w/o $\mathcal{L}_{\mathrm{info}}$)
shows how the contribution of simulator-labeled information-preservation
supervision changes with misspecification strength. The
performance gap between the two variants became larger under stronger damping.
As the simulator--reality gap increases, marginal alignment can reduce domain
discrepancy, but it does not ensure preservation of the relation between the
physical parameter and the posterior input
\citep{johansson2019support,zhao2019learning,stojanov2021domain}. This trend is
therefore consistent with the intended role of $\mathcal{L}_{\mathrm{info}}$,
which becomes most useful when misspecification is large enough.

\subsection{Training dynamics of the information-preservation objective}

Figure~\ref{fig:info_dynamics} provides an additional analysis of the training
dynamics of $\Linfo$ in Eq.~\eqref{eq:linfo_new}. We compare SPIN with its
transport-only variant, SPIN (w/o $\mathcal{L}_{\mathrm{info}}$), where the same
quantity is monitored but not used as a transport constraint.

The difference between the two variants was smallest on weakly misspecified SIR
and larger on Pendulum and on the controlled physical simulator--reality gaps in
Wind Tunnel and Light Tunnel. This trend is consistent with the main results.
The effect of information-preservation supervision becomes more visible as the
simulator--reality gap increases. The training dynamics
therefore provide supporting evidence for the improvements observed in
Table~\ref{tab:main_results} and Figure~\ref{fig:misspec_strength}. The
simulator-side NPE loss $\Ls$ is reported in
Appendix~\ref{app:optimization_dynamics}.


\begin{figure}[t]
    \centering
    \includegraphics[width=0.92\linewidth]{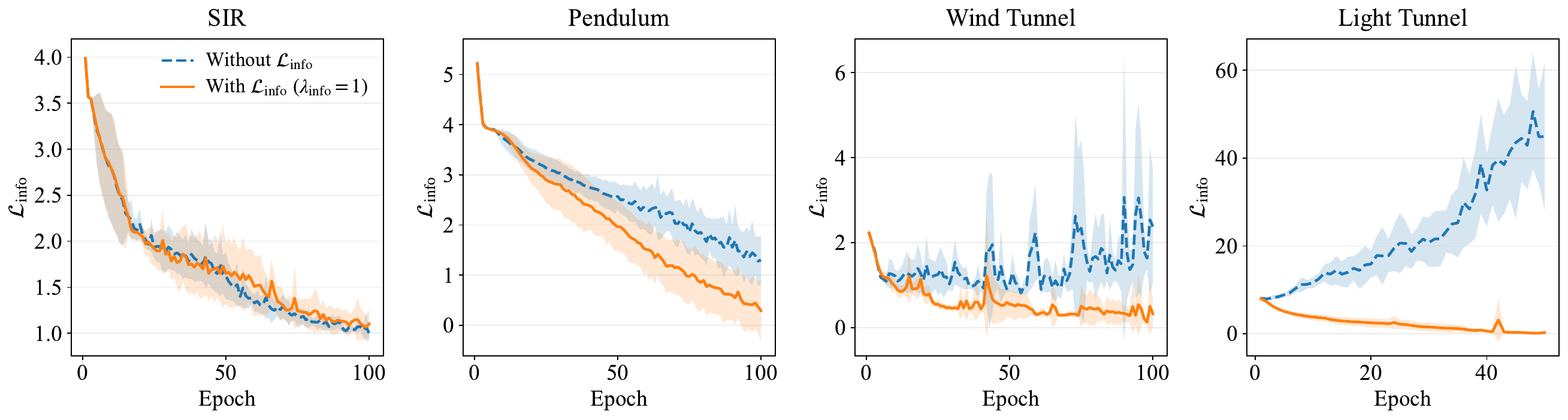}
    \caption{\textbf{Effect of information-preservation supervision during training.}
    We plot the mean over five runs together with run-to-run variability for
    $\Linfo$ in Eq.~\eqref{eq:linfo_new}, comparing the transport-only variant
    $(\lambda_{\mathrm{info}}=0)$ with the full method
    $(\lambda_{\mathrm{info}}=1)$. The difference is smallest on SIR and becomes larger on
Pendulum, Wind Tunnel, and Light Tunnel.}
    \label{fig:info_dynamics}
\end{figure}

\subsection{Sensitivity to the amount of unlabeled real-world data}

We evaluated how SPIN changes with $N_r$, the number of unlabeled real-world
observations used during adaptation. Figure~\ref{fig:nreal_sensitivity} shows
that the effect of the real-world data budget depends on the task and metric.
SIR and Wind Tunnel showed relatively small variation across budgets. Pendulum
showed clearer changes in LPP and ACAUC, whereas Light Tunnel was more sensitive
to very small real-world budgets, especially in RMSE and LPP.

This pattern is consistent with the role of unlabeled real-world observations in
SPIN. The real-world pool provides samples from the target observation
distribution used to train simulator-to-real transport. In particular,
$G_{sr}$ is trained so that simulator-originated observations are translated
toward the empirical real-world distribution, which in turn defines the
real-like inputs on which the return transport $G_{rs}$ is constrained. When
$N_r$ is very small, this empirical target distribution can be estimated only
coarsely, which can make transport learning less stable, especially under larger
simulator--reality gaps. Together, these results indicate that sufficient
coverage of the real-world observation distribution is important for stable
transport learning, but that posterior quality remains task- and
metric-dependent rather than being determined by $N_r$ alone.

\begin{figure}[t]
    \centering
    \includegraphics[width=0.92\linewidth]{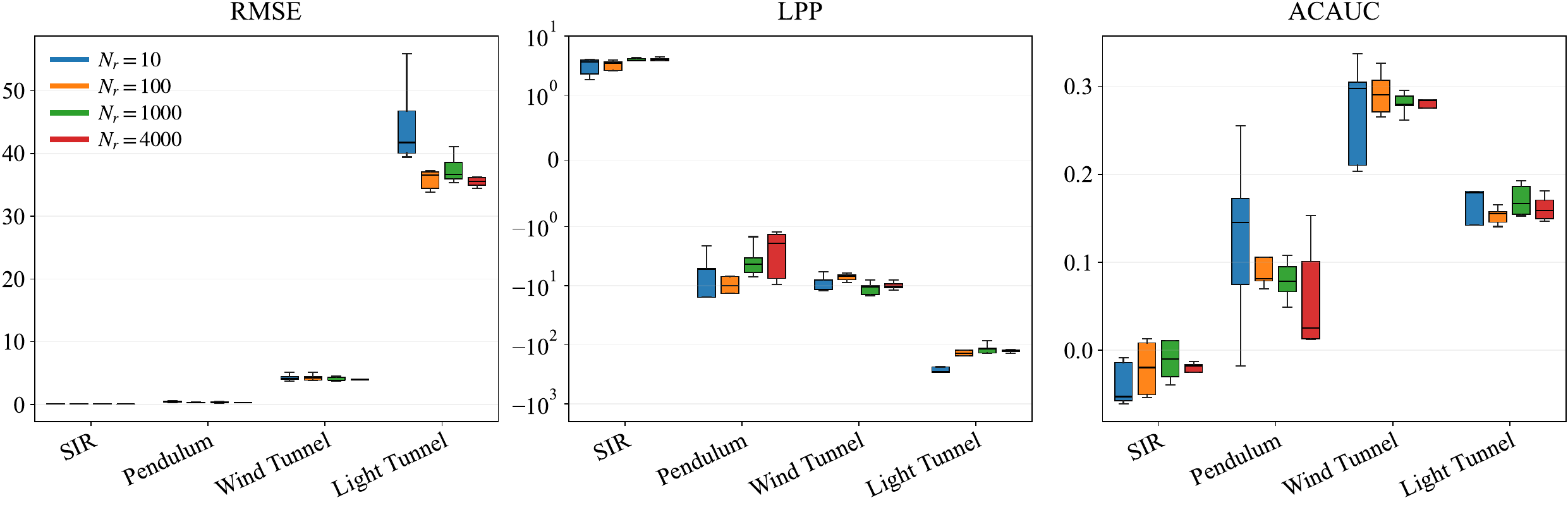}
    \caption{\textbf{Performance across unlabeled real-world data budgets.}
We vary $N_r$, the number of unlabeled real-world observations used during
adaptation, while keeping $N_s$ fixed, and report performance across tasks over
five independent runs. The effect of
the real-world data budget depends on the task and metric.}
    \label{fig:nreal_sensitivity}
\end{figure}

\subsection{Limitations}

Several limitations should be considered. First, if the learned transport poorly
approximates the parameter-preserving joint-transport condition in
Eq.~\eqref{eq:joint_sim_to_real_condition}, real-world inference after transport
can be degraded. This issue is related to how well the unlabeled real-world pool
represents the target observation distribution for simulator-to-real transport.
With small $N_r$, this empirical distribution can be estimated only coarsely,
especially under larger simulator--reality gaps, which can make transport
learning less stable. Second,
$\lambda_{\mathrm{info}}=1$ provides a label-free equal-weight default between
original and transported simulator observations, but the optimal weight may be
task-dependent. Sensitivity to this choice is reported in
Appendix~\ref{app:lambda_sensitivity}. Finally, SPIN adds training-time cost
through transport and discriminator updates, while preserving amortized inference
at test time.
Computational costs are summarized in Table~\ref{tab:compute_cost}.


\section{Conclusion}

We introduced SPIN for misspecified SBI with unlabeled,
unpaired real-world observations. SPIN learns information-preserving
bidirectional domain transport and performs posterior inference after mapping
real-world observations to the simulator domain. Instead of relying only on marginal domain alignment, SPIN encourages
domain transfer that preserves MI between the parameter of
interest and the transported observations. Across synthetic and controlled physical benchmarks, SPIN improved point
accuracy, posterior density at the reference parameter used for evaluation, and
expected-coverage diagnostics, with the clearest gains appearing under stronger
misspecification. These results suggest that domain transport for misspecified SBI
may benefit from explicitly preserving parameter-relevant information, in
addition to reducing marginal discrepancy between simulated and real-world data.

\bibliographystyle{unsrtnat}
\bibliography{references}

\newpage
\appendix

\section{Mathematical details for parameter-relevant information preservation}
\label{app:info_math_details}

This appendix collects the population identities used to interpret the
information-preservation objective. The identities clarify when simulator-label
supervision on transported observations corresponds to the conditional log-risk
of the real-world test-time rule, and why the same log-score is aligned with the
information-preservation interpretation used in the main text.

\subsection{Population identity for the test-time rule}

Let $(\Theta,X_s)\sim P_s$ denote the simulator joint distribution on
$\Theta\times\mathcal{X}_s$, and let $(\Theta,X_r)\sim P_r$ denote the
real-world joint distribution on $\Theta\times\mathcal{X}_r$. During training,
real-world samples are observed only through $X_r$; the variable $\Theta$ under
$P_r$ is used only to define population risks and evaluation quantities.

For any posterior estimator $q$, define the real-world test-time conditional
log-risk
\begin{equation*}
\mathcal{R}_{r}(q,G_{rs})
=
\mathbb{E}_{(\Theta,X_r)\sim P_r}
\left[
-\log q(\Theta\mid G_{rs}(X_r))
\right],
\end{equation*}
and the simulator-labeled transported risk
\begin{equation*}
\mathcal{R}_{\mathrm{tr}}(q,G_{sr},G_{rs})
=
\mathbb{E}_{(\Theta,X_s)\sim P_s}
\left[
-\log q(\Theta\mid G_{rs}(G_{sr}(X_s)))
\right].
\end{equation*}
The information-preservation objective in Eq.~\eqref{eq:linfo_new} is
\[
\Linfo(q)=\mathcal{R}_{\mathrm{tr}}(q,G_{sr},G_{rs}).
\]

The population condition connecting the simulator-originated training route and
the real-world test-time route is the parameter-preserving joint-transport
condition
\begin{equation}
(\Theta,X_r)
\stackrel{d}{=}
(\Theta,G_{sr}(X_s)),
\qquad
(\Theta,X_s)\sim P_s.
\label{eq:app_joint_transport_condition}
\end{equation}
This condition is stronger than matching only the observation marginal. Marginal
alignment can make $G_{sr}(X_s)$ and $X_r$ similar as observations, but it does
not by itself specify whether the transported observation preserves its relation
to $\Theta$ \citep{johansson2019support,zhao2019learning}. The same distinction
applies after a summary network: marginally matching $h_\omega(G_{sr}(X_s))$ and
$h_\omega(X_r)$ does not imply joint matching with the parameter
\citep{johansson2019support,zhao2019learning,stojanov2021domain}.

\begin{proposition}[Risk identity under parameter-preserving joint transport]
\label{prop:app_risk_equivalence}
If Eq.~\eqref{eq:app_joint_transport_condition} holds, then for any posterior
estimator $q$ and any real-to-simulator transport $G_{rs}$,
\begin{equation*}
\mathcal{R}_{r}(q,G_{rs})
=
\mathcal{R}_{\mathrm{tr}}(q,G_{sr},G_{rs})
=
\Linfo(q).
\end{equation*}
\end{proposition}

\begin{proof}
By Eq.~\eqref{eq:app_joint_transport_condition}, the random pairs
$(\Theta,X_r)$ and $(\Theta,G_{sr}(X_s))$ have the same distribution. Therefore,
for any measurable function $f$ with finite expectation,
\[
\mathbb{E}_{(\Theta,X_r)\sim P_r}
\left[
f(\Theta,X_r)
\right]
=
\mathbb{E}_{(\Theta,X_s)\sim P_s}
\left[
f(\Theta,G_{sr}(X_s))
\right].
\]
Choosing
\[
f(\theta,x)=-\log q(\theta\mid G_{rs}(x))
\]
gives
\[
\mathbb{E}_{(\Theta,X_r)\sim P_r}
\left[
-\log q(\Theta\mid G_{rs}(X_r))
\right]
=
\mathbb{E}_{(\Theta,X_s)\sim P_s}
\left[
-\log q(\Theta\mid G_{rs}(G_{sr}(X_s)))
\right].
\]
The left-hand side is $\mathcal{R}_{r}(q,G_{rs})$, and the right-hand side is
$\mathcal{R}_{\mathrm{tr}}(q,G_{sr},G_{rs})=\Linfo(q)$.
\end{proof}

Proposition~\ref{prop:app_risk_equivalence} is a population identity under the
stated joint-transport condition. It does not assert that a learned finite-data
transport necessarily satisfies this condition. In the main experiments, this
condition is encouraged by adversarial translation, cycle consistency, identity
regularization, and the simulator-label information-preservation objective, and
the resulting inference quality is evaluated empirically.

\subsection{Log-score objective and information preservation}
\label{app:LPPeq}

The information-preservation objective is a simulator-label log-score on the
transported simulator observation. Applying the standard log-score
decomposition to the posterior input \(G_{rs}(G_{sr}(X_s))\) gives
\begin{align*}
\Linfo(q)
&=
\mathbb{E}_{(\Theta,X_s)\sim P_s}
\left[
-\log q\!\left(\Theta\mid G_{rs}(G_{sr}(X_s))\right)
\right] \\
&=
H_s\!\left(\Theta\mid G_{rs}(G_{sr}(X_s))\right)
+
\mathbb{E}
\left[
\KL\!\left(
p_s\!\left(\Theta\mid G_{rs}(G_{sr}(X_s))\right)
\,\|\, 
q\!\left(\Theta\mid G_{rs}(G_{sr}(X_s))\right)
\right)
\right].
\end{align*}
Here, the conditional distribution is induced by
\((\Theta,X_s)\sim P_s\). Since \(H_s(\Theta)\) is fixed by the simulator prior,
\[
I_s\!\left(\Theta;G_{rs}(G_{sr}(X_s))\right)
=
H_s(\Theta)
-
H_s\!\left(\Theta\mid G_{rs}(G_{sr}(X_s))\right).
\]
Thus, minimizing the optimal simulator-label log-score on
\(G_{rs}(G_{sr}(X_s))\) reduces the conditional entropy of the parameter given
the transported simulator observation. With finite model capacity, \(\Linfo\)
serves as a supervised log-score surrogate for preserving parameter-relevant
information after transport.

The same log-score interpretation applies to the real-world test-time input.
For SPIN, real-world inference is performed from \(G_{rs}(X_r)\), so the
corresponding population log-risk is
\begin{align*}
\mathcal{R}_{r}(q,G_{rs})
&=
\mathbb{E}_{(\Theta,X_r)\sim P_r}
\left[
-\log q\!\left(\Theta\mid G_{rs}(X_r)\right)
\right] \\
&=
H_r\!\left(\Theta\mid G_{rs}(X_r)\right)
+
\mathbb{E}
\left[
\KL\!\left(
p_r\!\left(\Theta\mid G_{rs}(X_r)\right)
\,\|\, 
q\!\left(\Theta\mid G_{rs}(X_r)\right)
\right)
\right].
\end{align*}
Therefore,
\[
H_r(\Theta)
+
\mathbb{E}_{(\Theta,X_r)\sim P_r}
\left[
\log q\!\left(\Theta\mid G_{rs}(X_r)\right)
\right]
\le
I_r\!\left(\Theta;G_{rs}(X_r)\right),
\]
with gap
\[
\mathbb{E}
\left[
\KL\!\left(
p_r\!\left(\Theta\mid G_{rs}(X_r)\right)
\,\|\,
q\!\left(\Theta\mid G_{rs}(X_r)\right)
\right)
\right].
\]
This is the population quantity approximated by LPP \citep{lueckmann2021benchmarking} on the real-world test set. Within a task, \(H_r(\Theta)\) is
fixed, so comparing this expected log posterior density across methods is
equivalent to comparing the corresponding posterior-based information lower
bound up to an additive constant. We use this only as an interpretation of the
posterior log-score, not as an independent mutual-information estimator.

\newpage
\section{Dataset and benchmark details}
\label{app:datasets}

Across all tasks in the main experiments, we use the same split sizes for
simulator and real-world observations. Specifically, \(N_s=4000\) simulator
observations are used for training, with 500 additional simulator observations
used for validation and 500 for testing. Similarly, \(N_r=4000\) real-world
observations are used for training, with 500 additional real-world observations
used for validation and 500 for testing.

\paragraph{SIR.}
SIR is a weakly misspecified synthetic benchmark based on a discrete-time
stochastic susceptible--infected--recovered model \citep{kermack1927sir}. We fix the population size and initial infected count to
\[
N=1000,\qquad I_0=10,
\]
and sample parameters from
\[
\beta \sim \mathcal{U}[0.05,0.5],\qquad
\gamma \sim \mathcal{U}[0.05,0.5],
\]
subject to the constraint $\beta>\gamma$. The simulator records daily new infections for $T=365$ days. Transitions are sampled with binomial draws. The real-world domain is generated from the same parameter values, but the daily counts are modified by a simple reporting-delay mechanism: on Saturdays and Sundays, $5\%$ of the reported infections are delayed, and the delayed mass is added back on Monday. 

\paragraph{Pendulum.}
Pendulum is a synthetic benchmark in which the simulator is a frictionless pendulum integrated with an Euler solver, while the real-world domain introduces damping. For data generation, we first sample
\[
\omega_0 \sim \mathcal{U}\!\left[\frac{\pi}{10},\pi\right],\qquad
\phi_0 \sim \mathcal{U}\!\left[-\frac{\pi}{2},\frac{\pi}{2}\right],
\]
and an initial angular velocity
\[
\dot\theta_0 \sim \mathcal{U}[-1,1].
\]
The simulator uses zero damping, whereas the synthetic real-world data use
\[
\alpha \sim \mathcal{U}[0.05,0.5].
\]
The Euler integrator is run for $T=200$ output steps over $T_{\max}=10$ seconds with inner step size $\Delta t_{\mathrm{inner}}=0.01$, which yields $5$ inner Euler updates per reported observation step. Gaussian observation noise with standard deviation $0.05$ is added in both domains.

\paragraph{Wind Tunnel.}
Wind Tunnel uses the \texttt{wt\_intake\_impulse\_v1} dataset from Causal Chambers \citep{gamella2025causalchambers}, with experiment $\texttt{load\_out\_0.5\_osr\_downwind\_4}$.
Each real observation is constructed as one impulse response extracted from the downwind barometer trace. The real impulse rows are shuffled once with a fixed permutation and then split into train/validation/test subsets.

For simulation, we use the mechanistic A2C3 wind-tunnel model. We sample the hatch position uniformly:
\[
H \sim \mathcal{U}[0,45].
\]
The misspecification reflects the gap between the mechanistic A2C3 model and the
measured chamber response to fan loads and hatch position.

\paragraph{Light Tunnel.}
Light Tunnel uses the \texttt{lt\_camera\_v1} dataset from Causal Chambers \citep{gamella2025causalchambers}, with
experiment \texttt{uniform\_ap\_1.8\_iso\_500.0\_ss\_0.005}. For each real image,
the dataset provides light-source RGB values $(R,G,B)$ and polarizer angles
$(\mathrm{pol}_1,\mathrm{pol}_2)$. We convert the polarizer angles to
\[
\alpha = \cos^2\!\bigl((\mathrm{pol}_1-\mathrm{pol}_2)\pi/180\bigr)
\]
using Malus' law \citep{collett2005polarization}. Thus,
\[
R,G,B \in [0,255],
\qquad
\alpha\in[0,1].
\]

For the simulated domain, we use the F3 light-tunnel model and sample
\[
R,G,B \sim \mathcal{U}[0,255],
\qquad
\alpha \sim \mathcal{U}[0,1].
\]
The sampled dimming variable is converted to a polarizer-angle pair
$(0,\Delta)$ with
\[
\Delta = \arccos(\sqrt{\alpha}) \cdot \frac{180}{\pi}.
\]
The simulator then renders the image using the stored camera sensitivity curves,
white-balance coefficients, exposure value, and tunnel geometry constants. The misspecification reflects the gap between this simplified optical model and
the measured real light-tunnel response.

\newpage
\section{Implementation details and computational cost}
\label{app:implementation}

\paragraph{Baseline training.}
\textbf{NPE} \citep{papamakarios2016npe,greenberg2019automatic} trains the
summary network and conditional flow jointly by minimizing the negative
conditional log-likelihood on simulated pairs $(\theta,x_s)$.
\textbf{NPE-MMD} \citep{elsemuller2025uda} uses the same architecture and
optimizer as NPE and adds an inverse-multiquadratic MMD (IMQ-MMD) \citep{gretton2012mmd} penalty
between simulated and real summary features. We use the standard
multi-scale IMQ kernel \citep{tolstikhin2018wae} with scales
$\{0.1,0.2,0.5,1,2,5,10\}$. \textbf{NPE-DANN} \citep{elsemuller2025uda} also uses the same summary network
and posterior flow, and adds a domain classifier with $3$ hidden layers of width
$256$. Domain confusion is applied through a gradient reversal layer with the
standard schedule \citep{ganin2016dann}
\[
\lambda_{\mathrm{grl}}(p)=\frac{2}{1+\exp(-10p)}-1,
\]
where $p$ is the normalized training progress. The domain classifier is trained
jointly with the summary network and posterior flow using binary cross-entropy.

\paragraph{Parameterization and posterior family.}
For all tasks, the parameters are bounded. We therefore map parameters $\theta$
to an unconstrained latent variable $u$ with a componentwise logit transform and
train the posterior model in $u$-space. The posterior estimator is a conditional
masked autoregressive flow (MAF) implemented with \texttt{zuko}
\citep{papamakarios2017maf,rozet2023zuko}. For all tasks, the conditioner depth
is fixed to $3$. At evaluation time, posterior samples are transformed back to
the parameter space through the inverse sigmoid map. Table~\ref{tab:impl_hparams}
summarizes the main training hyperparameters.

\paragraph{Optimization.}
For NPE, NPE-MMD, and NPE-DANN, we optimize the summary network and posterior
flow with AdamW \citep{loshchilov2019adamw}. For NPE-DANN, the domain classifier
is optimized jointly with the same AdamW optimizer. For SPIN, we use three
optimizers: (i) Adam \citep{kingma2015adam} for the generators, (ii) Adam for
the discriminators, and (iii) AdamW for the summary network and posterior flow.
For the generator and discriminator optimizers, we use $\beta_1=0.5$ and
$\beta_2=0.999$. The generator/discriminator learning rate is fixed to $0.0002$
in all tasks, while the posterior optimizer uses the task-specific learning rate
from Table~\ref{tab:impl_hparams}.

\begin{table}[H]
\centering
\caption{\textbf{Task-specific training hyperparameters used in the main experiments.}}
\label{tab:impl_hparams}
\small
\begin{tabular}{lcccc}
\toprule
 & SIR & Pendulum & Wind Tunnel & Light Tunnel \\
\midrule
Epochs & 100 & 100 & 100 & 100 \\
Batch size $B$ & 100 & 100 & 100 & 10 \\
Learning rate & $5\times10^{-4}$ & $1\times10^{-4}$ & $5\times10^{-4}$ & $5\times10^{-4}$ \\
Weight decay & $1\times10^{-4}$ & $1\times10^{-4}$ & $1\times10^{-4}$ & $1\times10^{-4}$ \\
Embedding dim & 6 & 10 & 10 & 20 \\
\bottomrule
\end{tabular}
\end{table}

\newpage

\paragraph{Summary networks.}
The posterior flow is conditioned on a task-specific summary network
\citep{deistler2025practicalguide,radev2020bayesflow,chen2021nass}. For
Pendulum and Wind Tunnel, we use 1D convolutional summary networks following
encoder architectures used in prior misspecified SBI work
\citep{wehenkel2025rope,senouf2025frisbi}. For SIR, we use a compact 1D
convolutional summary network, and for Light Tunnel we use a 2D convolutional
summary network operating directly on RGB images. Table~\ref{tab:summary_arch}
lists the task-specific summary networks used in all methods
for each task.

\begin{table}[H]
\centering
\caption{\textbf{Task-specific summary networks.}}
\label{tab:summary_arch}
\small
\begin{tabular}{p{2.6cm}p{9.9cm}}
\toprule
Task & Architecture \\
\midrule
SIR &
\textbf{Summary network:} 1D CNN with conv stack
$1{\to}16{\to}64{\to}128{\to}128{\to}128$,
kernel sizes $(5,5,3,3,3)$, strides $(1,2,1,2,1)$, dilation $2$ in the last
three convolutions, average pooling after the 2nd and 4th convolutions, and
adaptive average pooling to length $8$. The head is an MLP
$\texttt{flat}\to256\to128\to6$ with ReLU activations.\\
\midrule
Pendulum &
\textbf{Summary network:} 1D CNN with conv stack
$1{\to}16{\to}64{\to}128{\to}128{\to}128{\to}128$,
kernel size $3$, dilation $2$, strided layers at the 2nd, 4th, and 6th
convolutions, and average pooling after the 2nd, 4th, and 6th convolutions. The
convolutional trunk is followed by an MLP
$\texttt{flat}\to512\to128\to32\to10$ with ReLU activations. \\
\midrule
Wind Tunnel &
\textbf{Summary network:} 1D CNN with conv stack
$1{\to}8{\to}32{\to}64{\to}64{\to}64{\to}64$,
kernel size $3$, dilation $2$, strided layers at the 2nd, 4th, and 6th
convolutions, average pooling after the 2nd, 4th, and 6th convolutions, and an
MLP $\texttt{flat}\to256\to64\to16\to10$ with ReLU activations. \\
\midrule
Light Tunnel &
\textbf{Summary network:} 2D CNN with conv stack
$3{\to}48{\to}96{\to}128$ using kernels $(5,3,3)$ with stride $2$ at every
convolution, followed by adaptive average pooling to $4\times4$, then an MLP
$\texttt{flat}\to128\to20$ with ReLU activations. \\
\bottomrule
\end{tabular}
\end{table}

\newpage

\paragraph{SPIN generators and discriminators.}
For SPIN, we use bidirectional observation-space translators $G_{sr}$ and
$G_{rs}$ together with one discriminator per domain \citep{zhu2017cyclegan}. For the 1D tasks, both
translators use the same architecture: a stem convolution $1{\to}32$ with kernel
size $7$, two downsampling convolutions $32{\to}64{\to}128$ with stride $2$, a
bottleneck composed of $4$ residual blocks with dilation schedule $(1,2,4,8)$,
and two upsampling stages with linear interpolation and $3\times1$ convolutions
$128{\to}64{\to}32$, followed by a $7\times1$ output convolution. The final
prediction is combined with the input through a learned residual skip:
\[
y = \sigma(a)\,x + (1-\sigma(a))\,\hat y.
\]
For the 2D image task, both translators use a 2D analogue of the same design
with channels $3{\to}32{\to}64{\to}128$, bilinear upsampling, and $4$ residual
bottleneck blocks.

For the 1D tasks, the discriminator trunk is
$1{\to}32{\to}64{\to}128{\to}128$ with LeakyReLU activations and a fused
patch/global head. For the 2D task, the discriminator trunk is
$3{\to}32{\to}64{\to}128{\to}128$ with the same patch/global fusion design.
Spectral normalization \citep{miyato2018spectral} is enabled in all
discriminators.

\newpage

\paragraph{Training algorithm.}
Algorithm~\ref{alg:spin_training} summarizes the SPIN training loop, including
the generator, discriminator, and posterior updates performed for each
mini-batch.

\begin{algorithm}[H]
\caption{\textbf{Training procedure for SPIN}}
\label{alg:spin_training}
\begin{algorithmic}[1]
\Require simulated training pairs $\{(\theta_i,x^s_i)\}$, unlabeled real training pool $\{x^r_j\}$, summary network $h_\omega$, posterior flow $q_\psi$, translators $G_{sr},G_{rs}$, discriminators $D_R,D_S$
\State Initialize $h_\omega,q_\psi,G_{sr},G_{rs},D_R,D_S$
\For{epoch $=1,\dots,E$}
    \For{mini-batch}
        \State Sample simulated batch $(\theta,x_s)$ and unlabeled real batch $x_r$
        \State \textbf{Generator step:}
        \State \hspace{1em} $x_{sr}\gets G_{sr}(x_s)$,\quad $x_{rs}\gets G_{rs}(x_r)$
        \State \hspace{1em} $x_{srs}\gets G_{rs}(G_{sr}(x_s))$,\quad $x_{rsr}\gets G_{sr}(G_{rs}(x_r))$
        \State \hspace{1em} Compute adversarial, cycle-consistency, and identity losses
        \State \hspace{1em} Compute information-preservation loss $-\log q_\psi(\theta\mid h_\omega(x_{srs}))$
        \State \hspace{1em} Form $\LG=\lambda_{\mathrm{adv}}\LadvG+\lambda_{\mathrm{cyc}}\Lcyc+\lambda_{\mathrm{id}}\Lid+\lambda_{\mathrm{info}}[-\log q_{\sg(\psi,\omega)}(\theta\mid x_{srs})]$
        \State \hspace{1em} Update $G_{sr},G_{rs}$ by descending $\nabla_{G_{sr},G_{rs}}\LG$
        \State \textbf{Discriminator step:}
        \State \hspace{1em} Form $\LD=\LadvD$ with hinge losses on real and translated samples
        \State \hspace{1em} Update $D_R,D_S$ by descending $\nabla_{D_R,D_S}\LD$
        \State \textbf{Posterior step:}
        \State \hspace{1em} Recompute $x_{srs}$ without gradient through the translators
        \State \hspace{1em} Compute
        \[
        \mathcal{L}_{\mathrm{NPE}}=
        -\log q_\psi(\theta\mid h_\omega(x_s))
        +
        \lambda_{\mathrm{info}}
        \bigl[-\log q_\psi(\theta\mid h_\omega(x_{srs}))\bigr]
        \]
        \State \hspace{1em} Update $h_\omega,q_\psi$
    \EndFor
    \State Evaluate simulator-side validation NPE loss on $(\theta_{\mathrm{val}},x^s_{\mathrm{val}})$
\EndFor
\State Use $q_\psi(\theta\mid x_{rs})$ for test-time inference
\end{algorithmic}
\end{algorithm}

\newpage

\paragraph{Compute environment.}
All experiments were run on a single NVIDIA RTX A6000 GPU unless stated
otherwise. SPIN is computationally heavier than NPE, NPE-MMD, and NPE-DANN
because each mini-batch includes generator, discriminator, and posterior updates.
This cost affects training only; test-time inference remains amortized and uses
a single real-to-simulator transport followed by posterior evaluation.

\paragraph{Computational cost.}
We compare the computational cost of all methods in Table~\ref{tab:compute_cost}.
The table reports trainable parameters, average training time per iteration, and
amortized inference time per observation. Training time includes all updates used
by each method, and inference time is measured for posterior evaluation with
1000 samples per observation.

\begin{table}[H]
\centering
\caption{\textbf{Model size and computational cost.}
We report trainable parameters, average training time per iteration, and
amortized inference time per observation. Training time includes all updates used
by each method. Inference time is measured for posterior evaluation with 1000
samples per observation.}
\label{tab:compute_cost}
\resizebox{\linewidth}{!}{%
\begin{tabular}{llccc}
\toprule
Task & Method & Trainable parameters & Training time / iteration & Inference time / obs. \\
\midrule
\multirow{4}{*}{SIR}
& NPE & 467,324 & 6.49 $\pm$ 0.04 ms & 3.79 $\pm$ 0.39 ms \\
& NPE-MMD & 467,324 & 11.67 $\pm$ 0.05 ms & 3.76 $\pm$ 0.30 ms \\
& NPE-DANN & 600,957 & 8.88 $\pm$ 0.05 ms & 3.58 $\pm$ 0.01 ms \\
& SPIN & 1,606,672 & 100.91 $\pm$ 0.09 ms & 5.32 $\pm$ 0.09 ms \\
\midrule
\multirow{4}{*}{Pendulum}
& NPE & 1,387,744 & 14.93 $\pm$ 0.35 ms & 5.59 $\pm$ 0.05 ms \\
& NPE-MMD & 1,387,744 & 24.86 $\pm$ 0.46 ms & 5.60 $\pm$ 0.04 ms \\
& NPE-DANN & 1,522,401 & 21.30 $\pm$ 0.31 ms & 5.65 $\pm$ 0.03 ms \\
& SPIN & 2,527,092 & 323.30 $\pm$ 81.20 ms & 6.97 $\pm$ 1.00 ms \\
\midrule
\multirow{4}{*}{Wind Tunnel}
& NPE & 112,250 & 7.25 $\pm$ 0.06 ms & 2.27 $\pm$ 0.19 ms \\
& NPE-MMD & 112,250 & 12.69 $\pm$ 0.03 ms & 2.13 $\pm$ 0.02 ms \\
& NPE-DANN & 246,907 & 10.10 $\pm$ 0.04 ms & 2.14 $\pm$ 0.03 ms \\
& SPIN & 1,251,598 & 96.31 $\pm$ 0.56 ms & 3.88 $\pm$ 0.47 ms \\
\midrule
\multirow{4}{*}{Light Tunnel}
& NPE & 1,454,308 & 6.49 $\pm$ 0.16 ms & 8.54 $\pm$ 0.48 ms \\
& NPE-MMD & 1,454,308 & 13.40 $\pm$ 0.41 ms & 8.69 $\pm$ 0.66 ms \\
& NPE-DANN & 1,591,525 & 9.66 $\pm$ 0.86 ms & 7.96 $\pm$ 0.02 ms \\
& SPIN & 4,981,842 & 99.66 $\pm$ 0.34 ms & 9.27 $\pm$ 0.04 ms \\
\bottomrule
\end{tabular}}
\end{table}

\newpage
\section{Additional experimental results and analyses}

\subsection{Comparison of related methods}
\label{app:method_comparison}

Table~\ref{tab:method_comparison} summarizes the related methods discussed in
the main text.

\begin{table}[H]
\centering
\caption{\textbf{Comparison of related misspecified SBI methods.}
The table summarizes whether each method uses real-world data, whether it
requires a calibration set with known parameters, whether inference is fully
amortized, and the main mechanism used to handle model misspecification. Here,
a calibration set denotes real-world observations with known parameters.}
\label{tab:method_comparison}
\resizebox{\linewidth}{!}{%
\begin{tabular}{lcccc}
\toprule
Method & Real-world data & Calibration set & Amortized & Approach \\
\midrule
NPE \citep{papamakarios2016npe,greenberg2019automatic} & None & No & Yes & Simulator-only inference \\
RoPE \citep{wehenkel2025rope} & Yes & Yes & Not fully & Optimal transport \\
FRISBI \citep{senouf2025frisbi} & Yes & Yes & Yes & Optimal transport \\
NPE+SC \citep{mishra2025selfconsistency} & Yes & No & Yes & Bayesian self-consistency \\
NPE-RS \citep{huang2023robuststatistics} & Yes & No & Not fully & MMD-based domain alignment \\
NPE-MMD \citep{elsemuller2025uda} & Yes & No & Yes & MMD-based domain alignment \\
NPE-DANN \citep{elsemuller2025uda} & Yes & No & Yes & Adversarial domain alignment \\
SPIN & Yes & No & Yes & Information-preserving domain transfer \\
\bottomrule
\end{tabular}}
\end{table}

\newpage

\subsection{Posterior-discrepancy diagnostics}
\label{app:additional_quant}

Table~\ref{tab:posterior_discrepancy} reports secondary posterior-discrepancy
diagnostics corresponding to the analysis described in the main text.

\begin{table}[H]
\centering
\caption{\textbf{Posterior-discrepancy diagnostics.}
We report symmetric Kullback--Leibler divergence (sKL)
\citep{kullback1951information}, Wasserstein distance
\citep{ramdas2017wasserstein}, and posterior MMD
\citep{gretton2012mmd} as secondary diagnostics
computed between simulator-side posteriors and method-specific real-world
posteriors for matched simulated and real-world observations. On the more clearly misspecified Pendulum, Wind Tunnel, and Light Tunnel
benchmarks, SPIN reduces sKL and often improves the auxiliary discrepancy
measures, with task-dependent differences across Wasserstein distance and
posterior MMD.}
\label{tab:posterior_discrepancy}
\resizebox{\linewidth}{!}{
\begin{tabular}{llccccc}
\toprule
Task & Metric & NPE & NPE-MMD & NPE-DANN & SPIN (w/o $\mathcal{L}_{\mathrm{info}}$) & SPIN \\
\midrule
\multirow{3}{*}{SIR}
& sKL        & $2.82 \pm 0.13$ & $\mathbf{2.58 \pm 0.30}$ & $2.70 \pm 0.29$ & $3.65 \pm 0.88$ & $3.19 \pm 0.64$ \\
& Wasserstein & $0.0046 \pm 0.0003$ & $0.0041 \pm 0.0007$ & $\mathbf{0.0040 \pm 0.0005}$ & $0.0401 \pm 0.0029$ & $0.0046 \pm 0.0002$ \\
& MMD         & $0.0012 \pm 0.0000$ & $0.0012 \pm 0.0000$ & $0.0012 \pm 0.0000$ & $0.2011 \pm 0.0271$ & $\mathbf{0.0012 \pm 0.0000}$ \\
\midrule
\multirow{3}{*}{Pendulum}
& sKL        & $26.51 \pm 5.45$ & $25.59 \pm 7.81$ & $19.57 \pm 7.01$ & $10.13 \pm 2.70$ & $\mathbf{9.76 \pm 4.86}$ \\
& Wasserstein & $0.559 \pm 0.158$ & $0.284 \pm 0.075$ & $0.398 \pm 0.062$ & $0.229 \pm 0.027$ & $\mathbf{0.224 \pm 0.027}$ \\
& MMD         & $0.7310 \pm 0.0710$ & $0.5350 \pm 0.3590$ & $0.5430 \pm 0.1560$ & $\mathbf{0.3540 \pm 0.1020}$ & $0.4030 \pm 0.1410$ \\
\midrule
\multirow{3}{*}{Wind Tunnel}
& sKL        & $28.51 \pm 5.35$ & $14.80 \pm 5.13$ & $16.85 \pm 4.00$ & $19.31 \pm 4.39$ & $\mathbf{10.86 \pm 2.75}$ \\
& Wasserstein & $5.667 \pm 0.862$ & $\mathbf{2.968 \pm 0.576}$ & $3.269 \pm 0.377$ & $3.129 \pm 0.289$ & $3.277 \pm 0.508$ \\
& MMD         & $1.2430 \pm 0.1470$ & $0.7810 \pm 0.1870$ & $0.8540 \pm 0.1410$ & $0.9390 \pm 0.1330$ & $\mathbf{0.6640 \pm 0.1090}$ \\
\midrule
\multirow{3}{*}{Light Tunnel}
& sKL        & $37.31 \pm 6.33$ & $28.82 \pm 0.42$ & $26.72 \pm 3.86$ & $24.62 \pm 2.11$ & $\mathbf{21.38 \pm 1.01}$ \\
& Wasserstein & $51.158 \pm 12.456$ & $41.089 \pm 3.134$ & $29.013 \pm 8.909$ & $25.717 \pm 4.879$ & $\mathbf{17.868 \pm 0.823}$ \\
& MMD         & $1.0210 \pm 0.0960$ & $0.8020 \pm 0.0560$ & $0.5840 \pm 0.2520$ & $0.5070 \pm 0.1250$ & $\mathbf{0.3190 \pm 0.0250}$ \\
\bottomrule
\end{tabular}}
\end{table}

\newpage

\subsection{Evaluation on transported simulator observations}
\label{app:xsrs_supervised_eval}

Figure~\ref{fig:xsrs_eval_appendix} evaluates $x_{srs}$ using RMSE, LPP, and
ACAUC. Since $x_{srs}$ is generated from labeled simulator samples, this analysis
checks whether $\Linfo$ increases the posterior density assigned to the original
simulator parameter on the transported observations it directly supervises.

\begin{figure}[H]
    \centering
    \includegraphics[width=0.92\linewidth]{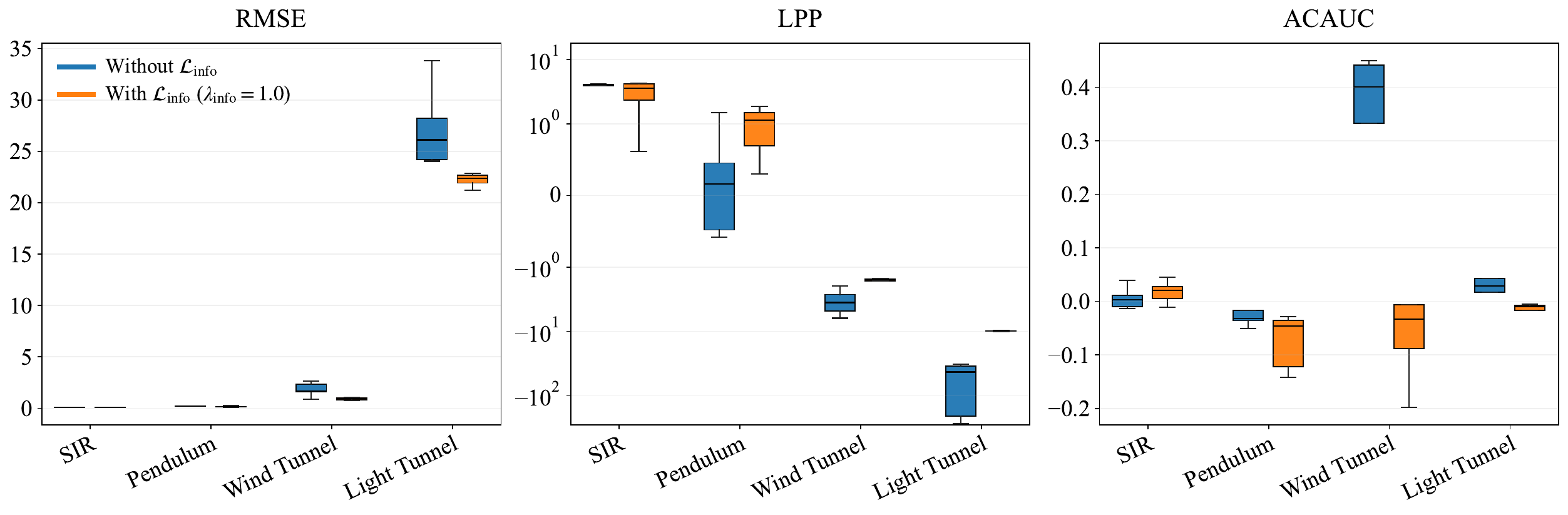}
    \caption{\textbf{Evaluation on transported simulator observations.}
    We evaluate $x_{srs}$ using RMSE, LPP, and ACAUC over
    five independent runs. Since $x_{srs}$ is generated from labeled simulator
    samples, the original simulator label remains available for evaluation. The comparison between SPIN and SPIN (w/o $\mathcal{L}_{\mathrm{info}}$)
therefore serves as a diagnostic for the simulator-originated training path,
showing what is gained by adding $\mathcal{L}_{\mathrm{info}}$ to the shared
bidirectional transport architecture.}
    \label{fig:xsrs_eval_appendix}
\end{figure}

\newpage

\subsection{Optimization dynamics}
\label{app:optimization_dynamics}

Figure~\ref{fig:ls_dynamics_appendix} reports the simulator-side NPE risk $\Ls$
for $\lambda_{\mathrm{info}}=0$ and $\lambda_{\mathrm{info}}=1$.

\begin{figure}[H]
    \centering
    \includegraphics[width=0.92\linewidth]{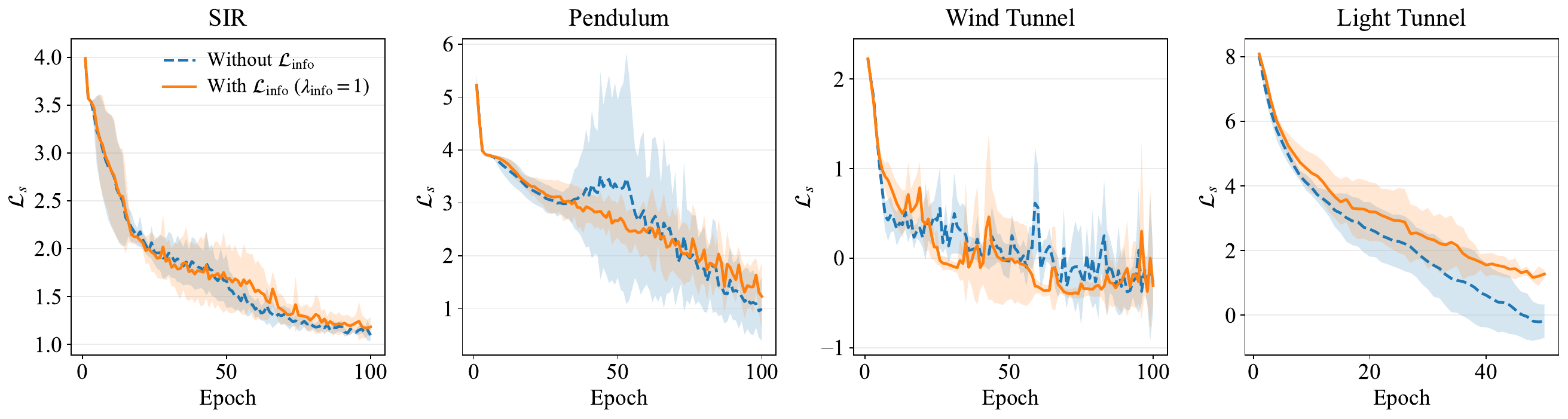}
    \caption{\textbf{Simulator-side risk dynamics.}
    We plot the median over five runs together with run-to-run variability for
    the simulator-side NPE risk $\Ls$, comparing $\lambda_{\mathrm{info}}=0$ and
    $\lambda_{\mathrm{info}}=1$. This analysis shows that adding
    information-preservation supervision preserves the simulator-side posterior
    training anchor.}
    \label{fig:ls_dynamics_appendix}
\end{figure}

\newpage

\subsection{Calibration analysis}
\label{app:calibration_analysis}

Figure~\ref{fig:calibration_appendix} reports the calibration curves
corresponding to the ACAUC values used in the main results
\citep{wehenkel2025rope,senouf2025frisbi,cook2006validation,talts2018sbc}.

\begin{figure}[H]
    \centering
    \includegraphics[width=0.92\linewidth]{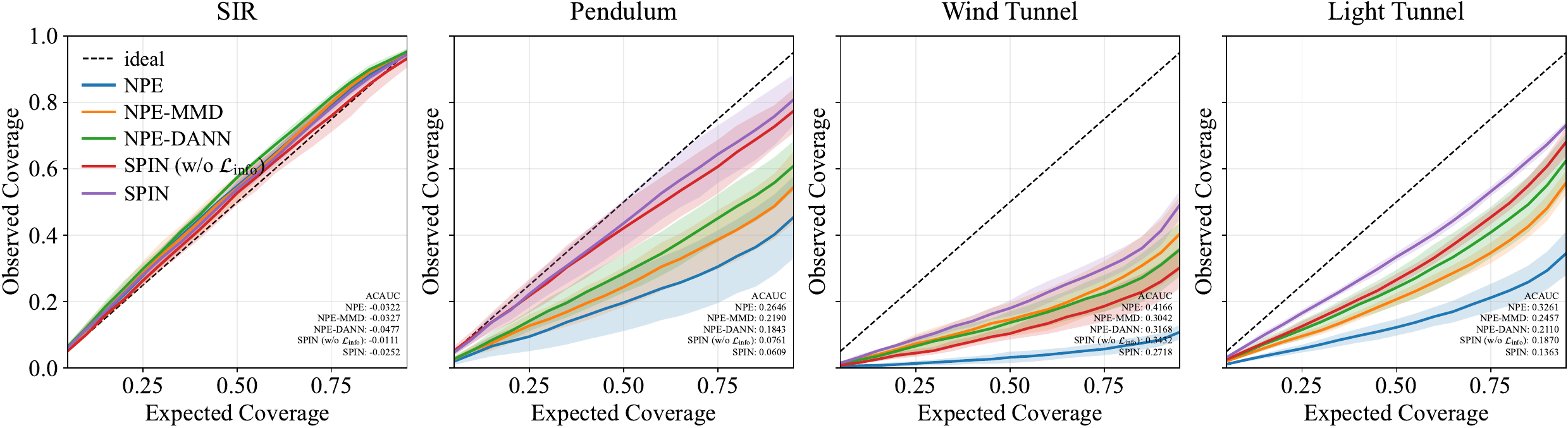}
    \caption{\textbf{Calibration analysis.}
    Marginal calibration curves and ACAUC values for the four tasks. Curves
    closer to the diagonal indicate better empirical coverage. SPIN improves or
    maintains calibration on the more strongly misspecified benchmarks, while
    differences are smaller on weakly misspecified SIR where the
    simulator-trained posterior is already competitive.}
    \label{fig:calibration_appendix}
\end{figure}

\newpage

\subsection{Sensitivity to information-preservation strength}
\label{app:lambda_sensitivity}

Figure~\ref{fig:lambda_info_sensitivity} varies the strength of
information-preservation supervision while keeping the rest of the training
setup fixed.

\begin{figure}[H]
    \centering
    \includegraphics[width=0.92\linewidth]{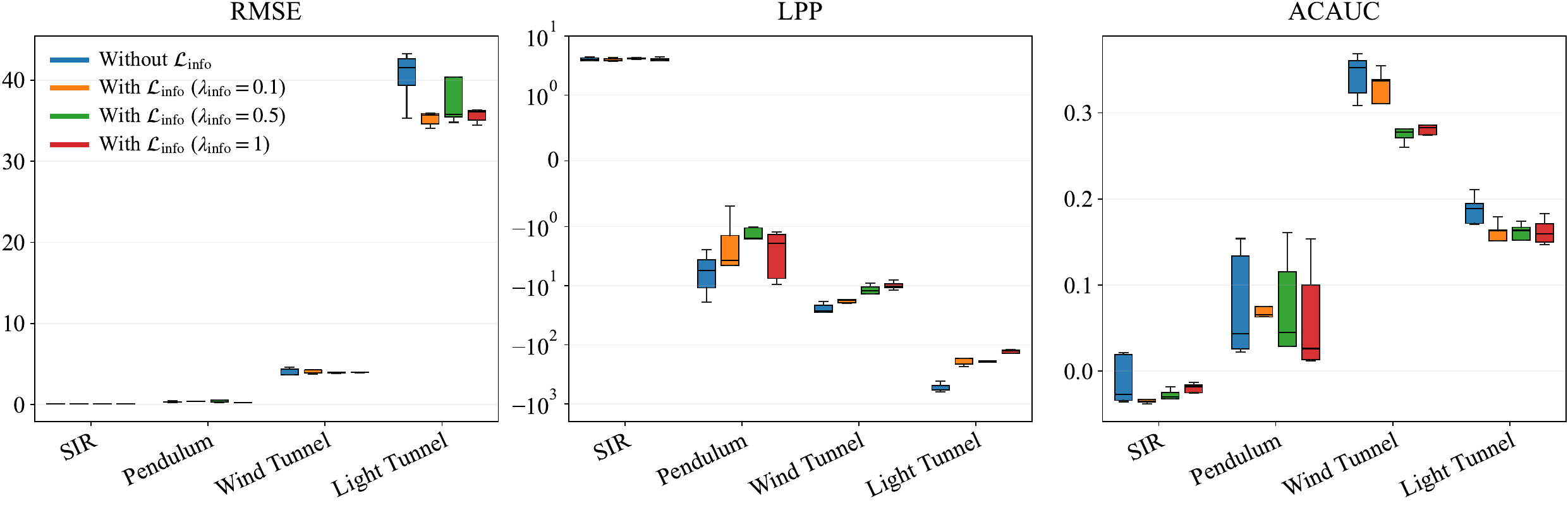}
    \caption{\textbf{Sensitivity analysis of $\lambda_{\mathrm{info}}$.}
    We report the effect of varying $\lambda_{\mathrm{info}}$ across tasks.
    Unlike Figure~\ref{fig:info_dynamics}, which isolates the presence or
    absence of the information-preservation term, this figure shows that the
    most effective information-preservation strength can be task-dependent. In
    particular, weakly misspecified settings such as SIR may require little or
    only moderate information-preservation pressure, whereas more clearly
    misspecified tasks benefit more consistently from stronger
    information-preservation constraints.}
    \label{fig:lambda_info_sensitivity}
\end{figure}

\newpage

\subsection{Qualitative transport examples}
\label{app:qualitative_transport_examples}

Figure~\ref{fig:qualitative_transport} visualizes the learned observation-space
transport and is included only as qualitative support.

\begin{figure}[H]
    \centering
    \includegraphics[width=0.92\linewidth]{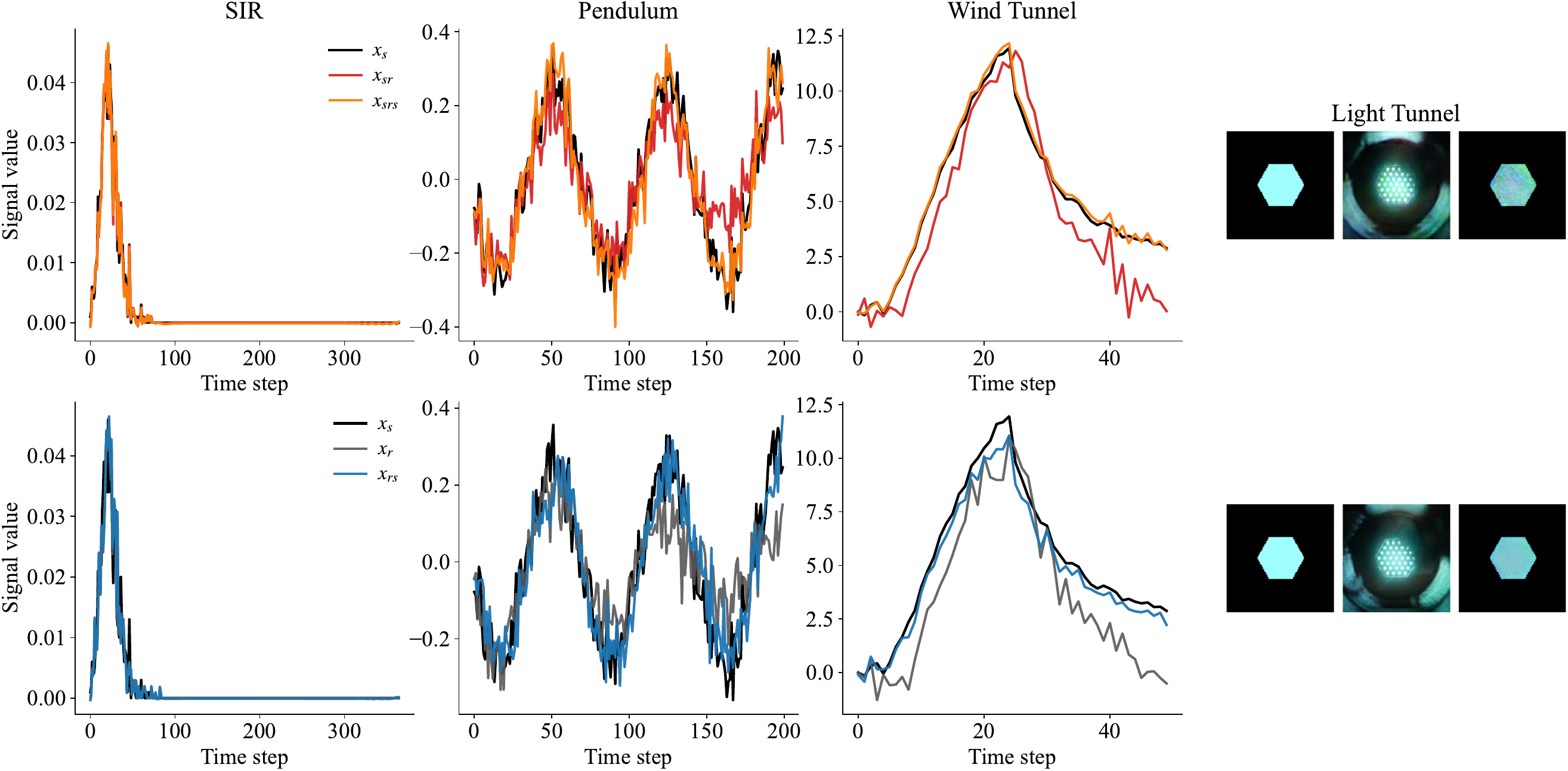}
    \caption{\textbf{Additional qualitative transport examples.}
    Representative examples of simulated observations $x_s$, real-world
    observations $x_r$, simulator-to-real transported observations $x_{sr}$, and
    real-to-simulator transported observations $x_{rs}$. These panels illustrate
    the learned observation-space transport, particularly on the more structured
    visual domain shifts. Visual alignment alone does not certify the joint
    parameter--observation alignment required by SPIN, but these examples provide
    qualitative support for the learned observation-space transport.}
    \label{fig:qualitative_transport}
\end{figure}

\clearpage

\end{document}